\documentclass{article}

\usepackage[margin=1in]{geometry}
\usepackage{graphicx,xcolor}
\usepackage{amsmath,amssymb,amsthm, algorithm}
\usepackage{hyperref}
\usepackage{authblk}
\usepackage{fancyhdr}
\usepackage{titlesec}
\usepackage{enumitem}

\usepackage{algpseudocode} 

\algrenewcommand\algorithmiccomment[1]{\hfill\(\triangleright\) #1}

\definecolor{sslblue}{RGB}{3,121,219} 
\hypersetup{
    colorlinks=true,
    linkcolor=sslblue,
    urlcolor=sslblue,
    citecolor=sslblue
}

\usepackage[T1]{fontenc}
\usepackage{lmodern}
\titleformat{\section}{\large\bfseries\color{sslblue}}{\thesection}{1em}{}
\titleformat{\subsection}{\normalsize\bfseries\color{sslblue}}{\thesubsection}{0.75em}{}

\newtheoremstyle{ssltheorem}
  {1em}{1em}{\itshape}{}{\color{sslblue}\bfseries}{.}{0.5em}{}
\theoremstyle{ssltheorem}
\newtheorem{theorem}{Theorem}
\newtheorem{lemma}{Lemma}
\newtheorem{assumption}{Assumption}
\newtheorem{definition}{Definition}
\newtheorem{corollary}{Corollary}

\newcommand{\institutionlogo}{
    \includegraphics[height=3cm]{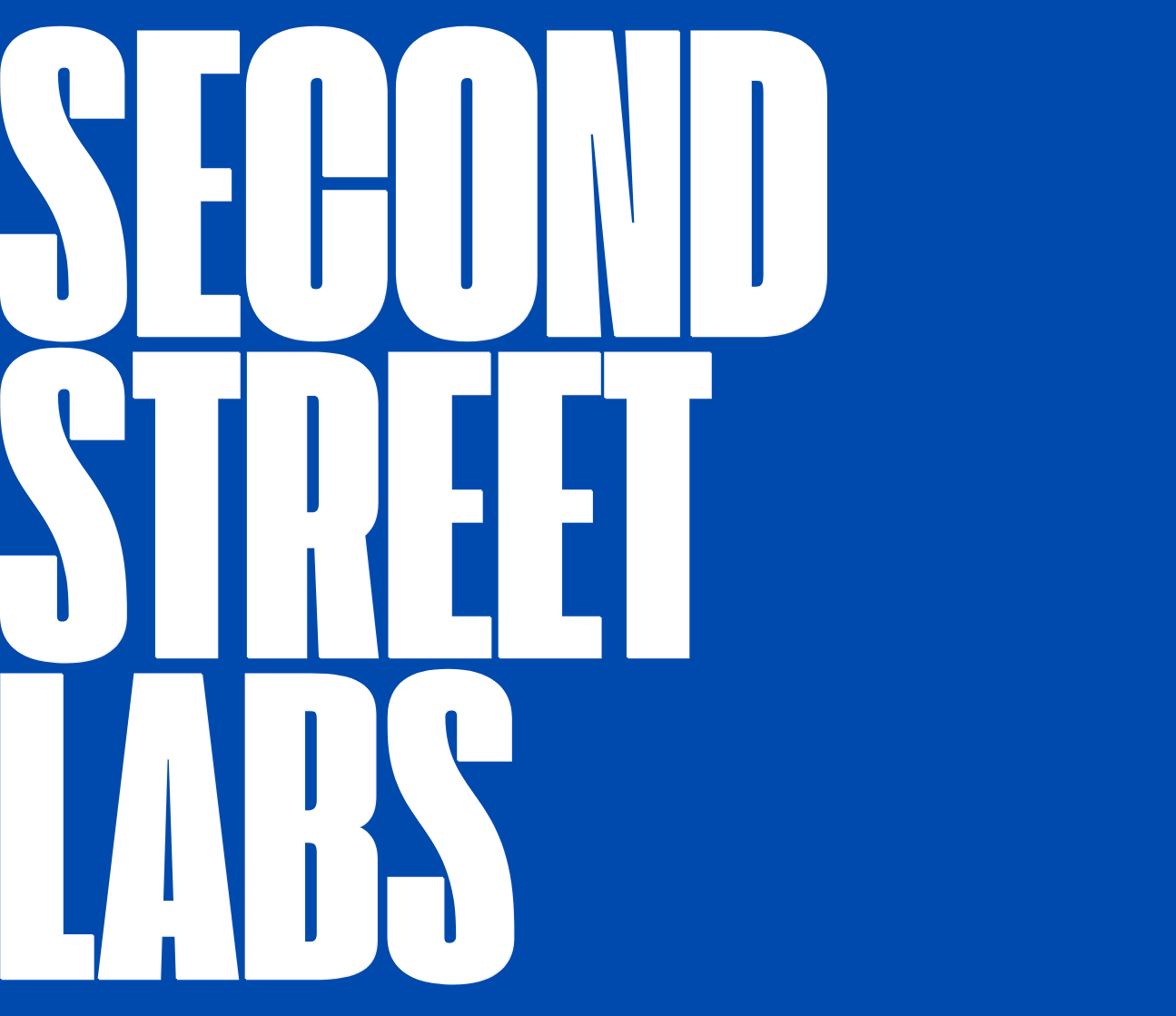}
}
\title{
    \institutionlogo\\[1cm]
    {\bfseries\color{sslblue} Mo' Memory, Mo' Problems: Stream-Native Machine Unlearning}
}
\author[1, 2]{Kennon Stewart}
\affil[1]{Second Street Labs;
Detroit, MI}
\affil[2]{University of Michigan; Ann Arbor, MI}
\affil[ ]{\texttt{kennon@secondstreetlabs.io}}
\date{\today}

\begin{document}
\maketitle

\begin{abstract}
Machine unlearning work assumes a static, i.i.d training environment that doesn't truly exist.
Modern ML pipelines need to learn, unlearn, and predict continuously on production streams of data.
We translate batch unlearning to the online setting using notions of regret, sample complexity, and deletion capacity.
We tighten regret bounds to a logarithmic $\mathcal{O}(\ln{T})$, a first for a machine unlearning algorithm.
When fitted with an online variant of L-BFGS optimization, the algorithm achieves state of the art regret with a constant memory footprint.
Such changes also extend the lifespan of an ML model before expensive retraining, making for a more efficient unlearning process.
\end{abstract}

\section{Introduction}
\label{sec:intro}

Industry has transitioned to a serverless and event-based cloud architecture for ML software.
Devices emit shards of data which are then streamed to a data center where models are built and inference is performed.
But the assumption of a cloud-based model is not guaranteed, and models are often embedded to the device itself to perform inference on the stream of data.
This brings statistical modeling directly to the device, but the question remains: how much data does the model need collect before performing valid inference?
The problem is made more interesting by recent legislation. The European Union's General Data Protection Regulation provides users the right to be forgotten: to have their data summarily deleted from a company's servers, and the effect of such data removed from a statistical model.
Statistical models need not only to learn on a stream of data but to also unlearn just as easily and without loss of accuracy.
We build specifically for the case of distributed systems where learning and unlearning are \textbf{necessary, interleaved} operations used to refine model accuracy.
We show that the memory pair meets state-of-the-art regret guarantees with a sublinear memory footprint particularly suited for online learning.
It also satisfies $(\varepsilon,\delta)$-unlearning requirements, enabling compliance with legislation like GDPR.
Unlearning quality is measured by a question: how closely does my unlearned model approximate the retrained ideal?
The frameworks of Sekhari et al. and Qiao et al.
provide provable $(\varepsilon,\delta)$ guarantees that a post-hoc deletion yields (in distribution) the same parameters as a fresh retrain \cite{Sekhari_Acharya_Kamath_Suresh_2021-03}.
However, both assume an offline model that never learns again after deletion.
We assume a stream environment where deletions and insertions are events processed in sequence.
\section{Stream-Native Learning}
\label{sec:stream-native}

The Memory Pair is an ordered pair of algorithms satisfying the following requirements.
We provide an example of such a pair of algorithms in our theoretical analysis.
\begin{itemize}
    \item \textit{Stream-native learning:} A learner $A$ ingests each example $(x_{t},y_{t})$ and performs an $O(d)$ update in micro-seconds, maintaining only lightweight sketches instead of per-sample gradients.
\item \textit{Deferred inference gate:} Predictions are withheld until the running sample count exceeds a theoretically derived sample-complexity threshold, ensuring the first answer is already PAC-competitive.
\item \textit{Symmetric unlearning:} A paired algorithm $\bar{A}$ accepts deletion requests via the same API and issues a one-step \textit{negative} update that preserves the learner’s regret bound.
The model is guaranteed to be within an arbitrary precision of the ideal cold-start model trained from scratch.
\item \textit{Live deletion-capacity odometer:} Each unlearning step depletes the model's deletion capacity budget.
When a model unlearns its capacity, the learner is flagged for retraining to preserve accuracy.
\end{itemize}
\subsection{Contributions}
\label{subsec:contributions}

We contribute a complete \emph{stream-native} learning-unlearning pipeline and its theoretical analysis.
Our contributions fall into two complementary categories: \emph{accuracy guarantees} and \emph{indistinguishability guarantees}.
\begin{enumerate}
    \item \textbf{Memory Pair framework.} We introduce the first online algorithm that couples a learner $A$ and an unlearner $\overline{A}$ into a unified "memory pair."
This pairing allows the system to process interleaved \textsc{insert}, \textsc{delete}, and \textsc{predict} operations on a live stream while storing only $\mathcal{O}(d\tau)$ curvature pairs.
\item \textbf{Accuracy guarantees.} Logarithmic cumulative regret guarantees are standard for online convex optimization problems.
We use the strong convexity assumption to achieve logarithmic bounds, a first for a certified unlearning algorithm:
    $$
    \mathcal{O}\!\left(\tfrac{G^{2}}{\lambda}\ln T\right).
$$ 

    \item \textbf{Capacity and complexity guarantees.} Using an AdaGrad-style statistic $S_{t}$, we derive closed-form formulas for $\gamma$-\emph{deletion capacity} and $\gamma$-\emph{sample complexity}.
These bounds tighten to $\widetilde{\mathcal{O}}(\sqrt{T})$ in the worst case and adapt automatically on benign data streams where gradients decay.
\item \textbf{Fidelity guarantees.} For every valid \textsc{delete} event, the unlearner provides $(\varepsilon,\delta)$-certified fidelity to the ideal retrain.
We embed these guarantees into a live \emph{capacity odometer}, which ensures the system halts and retrains precisely when further deletions would violate either fidelity or regret bounds.
\end{enumerate}

\paragraph{Design intuition.}  
Treating learning and unlearning as separate modules leads to duplicated state, inconsistent updates, and costly synchronization.
By pairing the learner and unlearner into a single \emph{memory pair}, both algorithms operate on a shared representation of curvature information.
This symmetry ensures that every \textsc{insert} update provides the gradient and curvature data later needed for a \textsc{delete}, while every \textsc{delete} respects the regret structure established during learning.
Our goal is for learning and unlearning to complement rather than interfere with one another, yielding tighter theoretical guarantees and near-instantaneous deletion updates in practice.
\section{Preliminaries}
\label{sec:prelim}

The batch unlearning framework from Sekhari et al. provides the theoretical inspiration for true online unlearning \cite{Sekhari_Acharya_Kamath_Suresh_2021-03}.
By focusing on population risk, they establish that unlearning with generalization guarantees is possible.
However, their algorithm is designed for a static, i.i.d. world.
Our Memory Pair framework adapts these core ideas for a dynamic, online setting. We replace the assumption of i.i.d.
data with a nonstationary stream, substitute the excess risk metric with the more robust notion of cumulative regret, and swap the expensive, offline Hessian-based update with a lightweight, stream-native L-BFGS approximation\cite{mokhtari2015online}.
The Memory Pair is not just a generalization of the batch algorithm but a \textit{necessary evolution} designed to handle the continuous learning, forgetting, and inference demands of a live data stream.
We start by framing our learners within the online learning paradigm.
Batch learning  minimizes the loss function evaluated over a static training set.
We instead use regret, which appropriately compares the online learner to the a comparator model for a particular realized sequence of events \cite{Cesa-Bianchi_Lugosi_2006}.
This demonstrates the key factor of the Sekhari paper. By requiring the model to satisfy population risk guarantees, the model is expected to perform well against some true but unknown parameter distribution.
This is very similar to the case of online learning because the model's evaluation is continuous, and the descent is adjusted accordingly.
\section{Related Work}

Our work lies at the intersection of (i) online quasi-Newton optimization—particularly limited-memory BFGS (L-BFGS)—and (ii) \emph{certified} machine unlearning.
On the optimization side, we adopt standard spectral regularity assumptions on the preconditioner maintained by online L-BFGS and time-decaying stepsizes under strong convexity.
On the unlearning side, we target distributional indistinguishability from retraining while operating in a streaming setting with interleaved insertions and deletions.
Our contribution is to \emph{pair} an online L-BFGS learner with a symmetric unlearner and to analyze regret \emph{including} deletion-induced noise.
\subsection{Online (Quasi-)Newton Methods and L-BFGS}
Stochastic and online quasi-Newton methods date back at least to \cite{schraudolph2007stochastic}, who introduced stochastic BFGS/L-BFGS updates for online convex optimization.
Robust, scalable stochastic quasi-Newton (SQN) variants were developed by \cite{byrd2016stochastic}, emphasizing curvature subsampling and practical damping.
For strongly convex smooth objectives, \cite{moritz2016linearly} obtained linear convergence rates for a stochastic L-BFGS scheme.
Closer to our assumptions, \cite{mokhtari2015online} established global convergence of \emph{online} L-BFGS under uniform eigenvalue bounds on the inverse-curvature matrices, which motivates our $(c,C)$ spectrum constraints and limited-memory window.
\subsection{Machine Unlearning and Certification}
A practical systems perspective on data deletion emerged with SISA training \cite{bourtoule2021sisa}, which shards and isolates training to reduce retraining cost after deletions;
our approach differs in being stream-native with paired insert/delete updates.
Formal treatments of efficient deletion began with \cite{ginart2019making2}, while \cite{guo2020certified} defined \emph{certified data removal}—indistinguishability from retraining—which we adopt as our target guarantee.
For long adversarial sequences of updates, \cite{neel2021descent} provided gradient-based unlearning with steady-state error and per-deletion runtime bounds;
\cite{gupta2021adaptive} handled adaptivity using tools from differential privacy and max-information.
Recent second-order directions include Hessian-free \emph{online} certified unlearning \cite{qiao2024hessianfree}, which avoids explicit Hessians and maintains per-sample statistics;
we instead use an L-BFGS memory with paired symmetric updates and provide regret bounds that explicitly include deletion noise.
Complementing algorithmic work, \cite{vanwaerebeke2025when} analyze when unlearning beats full retraining, identifying regimes via a phase diagram;
these complexity trade-offs motivate our capacity/accounting view. Finally, emerging efforts are extending certification beyond convex models to deep networks \cite{zhang2024towards}.
\paragraph{Our contribution relative to prior art.}
Unlike SISA-style retrain-on-shards methods, our \emph{memory-pair} pipeline performs local, symmetric quasi-Newton updates for both learning and unlearning within a single stream-native algorithm.
Compared to prior certified unlearning mechanisms—often tailored to ERM solutions or batch settings—we analyze \emph{online} regret under strong convexity and add an explicit deletion term to the bound, reflecting calibrated Gaussian perturbations used to certify removals.
On the optimization side, we adopt classical online L-BFGS regularity but combine it with a deletion-capacity accountant that governs when inference is deferred and how many deletions can be admitted before retraining becomes necessary.
\begin{definition}[Online learner with vanishing average regret]
\label{def:online-learner}
Let $\mathcal W \subseteq \mathbb{R}^d$ be a convex hypothesis set and $\{\ell_{t}\}_{t=1}^{\infty}$ an arbitrary sequence of loss functions $\ell_{t} : \mathcal {W} \to \mathbb{R}_{\ge 0}$.
An algorithm $\mathcal{A}$ that, after observing an outcome $(x_{t},y_{t})$, outputs a weight vector $w_{t} = \mathcal A(x_{1:t},y_{1:t}) \in \mathcal{W}$ is called an \emph{online learner} iff for every outcome sequence $\{(x_{t},y_{t})\}_{t=1}^{\infty}$ the average regret
$$
  \frac{1}{n}\,R_{n}(\mathcal{A})
  \;:=\;
\frac{1}{n}
  \Bigl(
    \sum_{t=1}^{n}\ell_{t}(w_{t})
    -
    \min_{w\in\mathcal W}\sum_{t=1}^{n}\ell_{t}(w)
  \Bigr)
$$
converges to zero,
i.e.
$$
  \frac{1}{n}\,R_{n}(\mathcal{A})
  \;\xrightarrow[n\to\infty]{}\;0.
$$
\end{definition}

In the online setting, a model is expected not only to learn incrementally, but also process data deletion requests with some measure of guarantee.
The algorithm must yield an unlearned model that is statistically indistinguishable from the retrained ideal.
This tight, path-dependent relationship requires a unified framework where learning and unlearning procedures are tightly coupled.
The world is rarely stationary. Learning new distributions under concept drift is essential to maintaining regret guarantees.
We formalize the notion of a memory pair as a coupled algorithm designed to operate on a stream.
\begin{definition}[Memory Pair]
\label{def:memory-pair}

Let $\{E_{t}\}_{t=1}^{N}$ be an event stream where $E_{t}\in \{\texttt{insert}(x_{t},y_{t}),\texttt{delete}(u_{t})\}$.
A pair of algorithms $(A,\bar{A})$ with shared state $\theta_{t-1}$ acts as
$$
\text{learn step: }\theta_{t} = A(\theta_{t-1},E_{t}),\qquad
\text{unlearn step: }\bar\theta_{t} = \bar{A}(\theta_{t-1},E_{t}).
$$

Denote by $\tilde{\theta_{t}}$ the {\it ideal replay model}, i.e.\ the ideal model retrained from scratch without the offending data.
Fix an regret target $\gamma$, confidence $\delta$, fidelity budget
$(\varepsilon^{*},\delta^{*})$, and a deletion capacity
$m$.
Then $(A,\bar{A})$ is a
\emph{$(\gamma,m,\varepsilon^{*},\delta^{*})$-memory pair} if, for every
event stream and for every horizon $N$, the following hold:

\begin{enumerate}
\item \textbf{Fidelity.} The strength of the guarantees depends entirely on the context.
Stronger guarantees will require higher sample complexities and lower deletion capacities.
      $$
      \Pr\!\bigl[\tfrac{1}{N} R_{N}(\theta)\le\gamma\bigr]\;\ge\;1-\delta.
$$

\item \textbf{$(\varepsilon, \delta)$-Certified Unlearning.}  
    Let $\mathcal{M}_{\text{del}}$ denote the (randomized) model state produced after applying a delete operation,
    and let $\mathcal{M}_{\text{retrain}}$ denote the model state produced by retraining from scratch on the dataset with
    the deleted point removed.
We say the algorithm achieves $(\varepsilon,\delta)$\emph{-certified unlearning}
    if for every measurable set $S$,
    \[
    \Pr[\mathcal{M}_{\text{del}}\in S] \;\le\;
e^{\varepsilon}\Pr[\mathcal{M}_{\text{retrain}}\in S] + \delta.
    \]
    By post-processing invariance, the same guarantee holds for any function of the model state (e.g., predictions).
\item \textbf{m-deletion capacity.}  
      Let $D_{N}$ be the number of \textsc{delete} events up to $N$.
If $D_{N}\le m$, then Conditions 1 and 2 hold.
\end{enumerate}
\end{definition}

The algorithm combines the regret guarantees of online learning with the DP-style language that are ubiquitous in machine unlearning.
Bundling the learning and unlearning algorithms is a natural extension to a body of work that has increasingly found them inseparable.
Indeed, recent Hessian-free methods of machine unlearning explicitly use the learning process to gain gradient information that is later used to unlearn \cite{qiao2024hessianfree}.
\section{Memory Pair}
\label{sec:memory-pair}

\subsection{Newton-Step Approach to Machine Unlearning}

Machine unlearning has a rich tool set of algorithms to approximate the ideal retrained model.
The Newton Step optimization method is used to remove the influence of a set of unlearned points.
With careful noise calibration, the unlearned method is certifiably indistinguishable from the ideal cold-start retrained model \cite{Sekhari_Acharya_Kamath_Suresh_2021-03}.
The unlearning approximation draws its strength from the combined use of first and second order information.
The gradient for every point is stored, and the initial influence of the point is scaled by the curvature of the loss.
The state of the model is carefully "logged" at each point in order to efficiently approximate the retrained ideal.
There has also been recent work in Hessian-Free unlearning, which stores curvature information of the training data to create a suitable estimate of the inverse Hessian using Hessian Vector Products.
But with a memory footprint that scales linearly, it is not feasible for the online use case where $n>>d$.
\subsection{Quasi-Newton Approaches to Online Learning}

We specifically target the Hessian inversion operation of the Newton-Step update for our unlearning algorithm.
We replace the Hessian inversion with a quasi-Newton L-BFGS optimization algorithm, eliminating the need for the precompute.
In fact, it enables a broader prediction space entirely. With no precompute required for deletion, interleaved learning and unlearning operations can be processed sequentially as the events are read from a data stream.
This is especially attractive for the prospect of federated learning in memory-constrained environments. More on this in Future Work.
We efficiently unlearn the contaminated influence near-instantaneously by computing the current gradient with respect to the unlearned point.
The L-BFGS optimization stores a constant number of curvature pairs that are used to estimate the second-order information of the loss surface.
When the surface is well-behaved, meaning bounded Hessian eigenvalues, then the L-BFGS approximation has proven convergence with a constant storage requirement.
no-op for insertion)

\begin{algorithm}
\caption{Memory Pair Step ($\mathsf{op}\in\{\textsc{insert},\textsc{delete}\}$)}
\label{alg:pair-step}
\begin{algorithmic}
\Require Point $(x,y)$, model state $(\theta,\mathrm{lbfgs})$, step size $\alpha$, budgets $(\varepsilon_s,\delta_s)$, op $\in\{\textsc{insert},\textsc{delete}\}$
\State \textbf{assert} $(\mathsf{op}=\textsc{insert})$ \textbf{or} (deletions$<K$ \textbf{and} len(lbfgs)$>0$)
\State $g_{\text{pre}} \gets \nabla_{\theta}\mathcal{L}(\theta; x,y)$ \Comment{common: gradient at current params}
\State $d \gets \mathrm{lbfgs.direction}(g_{\text{pre}})$ \Comment{common: quasi-Newton direction}
\State $\sigma_{\text{op}} \gets \mathbf{1}[\mathsf{op}=\textsc{delete}] \cdot \mathrm{calibrate\_noise}(\|d\|_2,\varepsilon_s,\delta_s)$
\State $\eta \sim \mathcal{N}(0,\sigma_{\text{op}}^{2} I)$ \Comment{no noise for insertion unless DP-training}
\State $\mathrm{sign} \gets \begin{cases} +1 & \mathsf{op}=\textsc{insert} \\ -1 & \mathsf{op}=\textsc{delete}\end{cases}$
\State $\theta_{\text{tmp}} \gets \mathrm{signed\_update}(\theta,\mathrm{sign},d,\alpha)$ \Comment{$\theta \pm \alpha d$}
\State $\theta \gets \theta_{\text{tmp}} + \eta$ \Comment{common: optional post-update noise}
\State $g_{\text{post}} \gets \nabla_{\theta}\mathcal{L}(\theta;
x,y)$ \Comment{common: gradient after update}
\State $s \gets \theta - \theta_{\text{tmp}} + (\theta_{\text{tmp}} - (\theta \mp \alpha d))$ \Comment{= actual $\Delta\theta$ realized}
\State $y_{\text{vec}} \gets g_{\text{post}} - g_{\text{pre}}$ \Comment{common: curvature pair $y$}
\If{$\mathsf{op}=\textsc{insert}$}
  \State $\mathrm{lbfgs.add\_pair}(s, y_{\text{vec}})$ \Comment{standard L-BFGS update}
\Else
  \State \textit{(optional)} $\mathrm{lbfgs.maintain}(s, y_{\text{vec}},\text{mode}=\text{'downdate'})$ \Comment{keep memory stable}
  \State $\mathrm{odometer.spend}(\varepsilon_s,\delta_s)$;
deletions$\,{+}{=}\,1$
\EndIf
\end{algorithmic}
\end{algorithm}

\textbf{Symmetry of the Memory Pair.} Both operations share an identical control flow: compute a local gradient, obtain a quasi-Newton direction from the shared L-BFGS memory, apply a signed parameter update ($+\alpha d$ for insertion, $-\alpha d$ for deletion), optionally inject calibrated Gaussian noise, and then update curvature statistics using the realized $(s,y)$ pair.
The only semantic differences are (i) the sign of the step and (ii) budget accounting: deletion spends $(\varepsilon_s,\delta_s)$ in the odometer, whereas insertion does not.
This symmetry makes the unlearner a first-class mirror of the learner, ensuring the same curvature information powers both updates and simplifying both the analysis and the implementation.
\begin{algorithm}
\caption{Memory Pair \textsc{Insert}}
\label{alg:insertion-symmetric}
\begin{algorithmic}
\Require $(x,y),(\theta,\mathrm{lbfgs}),\alpha$
\State \textbf{call} \ref{alg:pair-step} with $\mathsf{op}=\textsc{insert}$ and $(\varepsilon_s,\delta_s)=(0,0)$
\end{algorithmic}
\end{algorithm}

\begin{algorithm}
\caption{Memory Pair \textsc{Delete}}
\label{alg:deletion-symmetric}
\begin{algorithmic}
\Require $(x,y),(\theta,\mathrm{lbfgs}),\alpha$, budgets $(\varepsilon_s,\delta_s)$
\State \textbf{call} \ref{alg:pair-step} with $\mathsf{op}=\textsc{delete}$ and given $(\varepsilon_s,\delta_s)$
\end{algorithmic}
\end{algorithm}

When chosen carefully, the noise acts as a shock absorber for the model.
It dampens the unlearning effect that would leak the deleted user's information.
Our variance is specifically chosen to outpace both the (1) Hessian-sketch error of L-BFGS and (2) the cumulative impact of all future deletions.
We can estimate the second because we explicitly limit the number of deletions with deletion capacity and ensure its compliance with our odometer.
The unlearning algorithm is defined analogously. It takes the point to be deleted and computes the gradient of the loss at the point of prediction.
The L-BFGS method is used to estimate the influence of that point on the historical trajectory of the algorithm (approximated using the curvature points) to estimate the same Hessian inversion performed by the Newton Step.
\subsection{Deletion Capacity Accounting via an Odometer}

A critical component of the Memory Pair framework is its ability to manage the trade-off between unlearning fidelity and model regret.
This is accomplished through a strict accounting mechanism that we term a \textbf{deletion capacity odometer}.
We model deletion-capacity as a dynamic amount that changes based on the stream.
Unlike static analyses, we don't envision that a model has a fixed deletion capacity after a static training period.
This odometer regulates the model's utility (i.e., its regret bounds) in response to deletions. It is important to note that $(\varepsilon,\delta)$-certified unlearning budgets are monotone and do not replenish; our "replenishment" language refers strictly to utility. As additional data insertions improve regret bounds (e.g., by increasing $N$), the system can tolerate more deletions before violating its target utility threshold.

By first proving the most conservative case of regret bounds, we then show that strong convexity and dynamic regret analyses allow for $\mathcal{O}(\sqrt{T})$ bounds on deletion capacity and sample complexity.
For the sake of simplicity, we use the Zero-Concentrated Differential Privacy definition to scale the amount of noise used for deletions to the influence of the unlearned point.
\subsection{Composition under zero-Concentrated DP}
\label{subsec:zcdp_composition}

\paragraph{Capacity as a {$\rho$}-budget.}
We use Zero-Concentrated Differential Privacy (zCDP) \cite{Bun_Steinke_2016} for internal accounting due to its tighter composition properties, then convert to the standard $(\varepsilon,\delta)$-certified guarantee for reporting.
Fix a total fidelity budget $\rho_{\text{tot}}>0$. The model is initialized with $m$ deletion capacity and we allocate the budget \emph{uniformly},
$$
  \rho_{\text{s}}\;=\;\frac{\rho_{\text{tot}}}{m},
$$
so that after at most \(m\) deletions the cumulative privacy loss is
$
  \sum_{j=1}^{m}\rho_{\text{s}}
  = \rho_{\text{tot}}
$
by the additive composition rule of zCDP.
\paragraph{Per-delete noise calibration.}
In each \textsc{delete}$(u)$ the algorithm adds Gaussian noise
$
  \eta\sim\mathcal{N}\bigl(0,\sigma_{\text{s}}^{2}\mathbf I_{d}\bigr)
$
with scale
$$
  \sigma_{\text{s}}
  \;=\;
\frac{S_{\text{step}}}{\sqrt{2\rho_{\text{s}}}},
$$
where $S_{\text{step}}$ is a global $\ell_2$-sensitivity bound on the unlearning update derived from the loss function's properties (e.g., $S_{\text{step}} \le G/\lambda$ under strong convexity). The runtime accountant tracks deletions and the spent fidelity
$$
  \rho_{\text{spent}} = \texttt{deletions\_so\_far}\times\rho_{\text{s}}.
$$

A new deletion is rejected with a RuntimeError once
$$
  \rho_{\text{spent}} + \rho_{\text{s}} > \rho_{\text{tot}}.
$$

\paragraph{Why a finite $m$ is necessary.}
Each deletion increases both (i) the cumulative zCDP loss and
(ii) the \emph{variance} of the model parameters through \(\sigma_{\text{s}}^{2}\).
Beyond a problem-dependent threshold, the injected noise dominates the
learning signal and jeopardizes the algorithm’s average-regret
bound \(\gamma\) (\ref{subsec:adaptive_capacity}).
The $m$-deletion capacity therefore matches the largest $m$ for which
Theorems \ref{thm:gamma-capacity} and \ref{thm:gamma-sample}
remain valid.
\section{Theoretical Evaluation}
\label{sec:theory}

Our theoretical analysis shows that state of the art regret bounds are achievable for online machine unlearning.
We defer to previous works for the proof of $\mathcal{O}(\sqrt{T})$ regret and instead focus on the case of strong convexity (and sketch the proof in the appendix).
We instead start from a stronger assumption of $\lambda$-strong convexity and work backward using the convexity shortcut to show that such a shift yields stronger $\mathcal{O}(\ln T)$

\subsection{Convergence Under $\lambda$-Strong Convexity}

Sublinear average regret guarantees are standard for machine unlearning work.
It relies on the assumption of general convexity with respect to your loss function.
We provide tighter logarithmic regret bounds by transitioning our convexity assumption to one of strong convexity.
\begin{definition}[Strong $\lambda$-Convexity]
    A differentiable function f is $\lambda$-strongly convex if 
    $$
    f(y) \geq f(x) + \nabla f(x)^{\top}(y-x)+ \frac{\lambda}{2}\|y-x\|^{2}
    $$
    for some $\lambda$ and all $x,y$.
\end{definition}

The primary difference between the assumptions of general convexity and $\lambda$-strong convexity is the quadratic term at the end of the definition.
This provides a quadratic lower bound on the growth of the function and gives us some stability in terms of the gradient.
We then adjust our choice of $\eta_{t}$. We go from a nonincreasing deletion schedule to one that is strictly decreasing.
This decreasing learning rate cancels out the majority of the quadratic penalty term.
We then recognize that the regret bounds for the $\lambda$-strongly convex function are the sum of a telescoping and harmonic sum with logarithmic regret bounds.
The full proof is included in the appendix, but such guarantees provide strong performance guarantees for the memory pair algorithm in strongly convex settings.
\begin{theorem}[Logarithmic cumulative regret with $m$ certified deletions]
\label{thm:log_regret}
Let $\{\ell_{t}\}_{t=1}^T$ be a sequence of $\lambda$-strongly convex, $G$-Lipschitz loss functions over a closed, convex domain $W$ of diameter $D$.
Assume the inverse-curvature preconditioners maintained by online \mbox{L-BFGS} satisfy the uniform eigenvalue bounds $cI \preceq B_{t} \preceq CI$ (see Assumption~\ref{assum:bounded-hessian}).
Use the time-decaying step size $\eta_{t}=\frac{1}{\lambda t}$ for $t\ge 1$.
Suppose the algorithm processes at most $m$ deletion requests at arbitrary times, and each deletion is implemented by an additive Gaussian perturbation with scale $\sigma_{\text{step}}$ (independent across deletions), calibrated to meet a target $(\varepsilon^{*},\delta^{*})$ certification budget via standard Gaussian-mechanism composition.
Then, for any static comparator $w^{*}\in\arg\min_{w\in W}\sum_{t=1}^T \ell_{t}(w)$, with probability at least $1-\delta^{*}-\delta_{B}$,
\[
R_{T}(m)\;:=\;\sum_{t=1}^{T}\bigl(\ell_{t}(w_{t})-\ell_{t}(w^{*})\bigr)
\;\le\;
\frac{G^{2}}{\lambda c}\,\bigl(1+\ln T\bigr)\;+\;\Delta_{m},
\]
where the deletion contribution satisfies
\[
\Delta_{m} \;:=\; m\,G\,\sigma_{\text{step}}\,\sqrt{2\ln(1/\delta_{B})}.
\]
In particular,
\[
R_{T}(m)\;=\;\mathcal{O}\!\Bigl(\tfrac{G^{2}}{\lambda}\ln T\;+\;m\,G\,\sigma_{\text{step}}\,\sqrt{\ln\tfrac{1}{\delta_{B}}}\Bigr).
\]
\end{theorem}

\textbf{The above result is novel for machine unlearning.} The strong convexity assumption and adaptive step size produce telescoping harmonic sums which can be bounded uniformly $\sum^{T}_{t=1}\frac{1}{t} \leq 1 + \ln{T}$.
\subsection{Dynamic Regret with a Pathwise Comparator}
\label{subsec:dynamic_regret}

In nonstationary data streams it is unrealistic to benchmark the learner
against a single best model.
Dynamic regret instead compares the learner to an \emph{oracle path}
$\{w_{t}^{*}\}_{t=1}^{T}$ that may drift over time.
Formally
$$
  R_{t}^{\mathrm{dyn}}
  \;=\;
\sum_{t=1}^{T}\bigl[\ell_{t}(w_{t})-\ell_{t}(w_{t}^{*})\bigr].
$$

\begin{definition}[Path-length of the comparator]
\label{def:path_length}
The \emph{path-length} of the comparator sequence is
$
  P_{t}
  =\sum_{t=2}^{T}\lVert w_{t}^{*}-w_{t-1}^{*}\rVert_{2}.
$
It quantifies the environment’s nonstationarity
and appears in all known lower bounds for dynamic regret
\footnote{See Appendix for a short proof that any algorithm suffers
$\Omega(P_{t})$ dynamic regret in the worst case.}
\end{definition}

We keep the strongly-convex learning rate
$$
  \eta_{t} = \tfrac{1}{\lambda t},
$$
which already yielded \(O(\ln T)\) \emph{static} regret.
The only difference is that we now allow the comparator to move,
contributing an additional first-order error term that scales with $P_{T}$, or a measure of the stream's drift.
\begin{theorem}[Dynamic regret under $\lambda$-strong convexity]
\label{thm:dyn_regret}
Let Assumptions~\ref{assum:lipschitz}, –\ref{assum:stable-lbfgs} hold
and let the loss sequence
$\{\ell_{t}\}_{t=1}^{T}$ be $\lambda$-strongly convex and $G$-Lipschitz.
Run Algorithm~\ref{def:memory-pair} with
$\eta_{t} = (\lambda t)^{-1}$.
Then for \emph{any} comparator path $\{w_{t}^{*}\}_{t=1}^{T}$ we have
$$
  R_{T}^{\mathrm{dyn}}
  \;\le\;
\underbrace{\frac{G^{2}}{\lambda c}\bigl(1+\ln T\bigr)}_{\text{static term}}
  \;+\;
  \underbrace{G P_{T}}_{\text{pathwise term}}
$$
so that
$
  R_{t}^{\mathrm{dyn}}
  =
  \mathcal{O}\!\bigl(\tfrac{G^{2}}{\lambda}\ln T + G P_{T}\bigr).
$
\end{theorem}

When $P_{T}=0$ we recover the static-comparator bound of
Theorem \ref{thm:log_regret}.  Conversely, whenever $P_{T}=\Omega(T)$ the linear term dominates, and Theorem \ref{thm:dyn_regret} matches known minimax lower bounds for
strongly convex online learning in drifting environments \cite{vanwaerebeke2025when}.
The memory-pair algorithm preserves its logarithmic dependence on time while adapting linearly to environmental drift.
In practice, $P_{t}$ is often \emph{sublinear} because real data tend to evolve smoothly, so the bound above remains sharply better than classical $\mathcal{O}(\sqrt{T})$ results under general convexity.
\subsection{Adaptive Deletion Capacity \& Sample Complexity}
\label{subsec:adaptive_capacity}

We move on to define our sample complexity and deletion capacity bounds.
These quantities, introduced in Sekhari et al.'s paper, were initially defined for a static batch training case.
The assumption was that ML engineers specify an average regret guarantee and maximum number of deletions, then work backward to find the minimum number of samples required.
The goal is to minimize the sample complexity for a particular deletion capacity, or to minimize training time to achieve the maximum number of deletions before accuracy loss.
But the assumption of a fixed deletion capacity doesn't fit into a stream-native paradigm.
The learner not only unlearns data after the initial training, but continues to learn in the workload following training.
It makes sense that the deletion capacity would be an adaptive quantity that is lowered when processing deletions and replenished when learning additional data.
\paragraph{AdaGrad–style geometry.} We borrow the cumulative squared-gradient statistic for an adaptive version of sample complexity and deletion capacity.
$$
  S_{t} \;=\;\sum_{t=1}^{T}\lVert g_{t}\rVert_2^{2}
$$
$$
  g_{t} = \nabla\ell_{t}(w_{t}).
$$
Following the proof in Appendix~0.2.4 we run online L-BFGS with the
\emph{data-driven} step size
$$
  \eta_{t} \;=\;\dfrac{D}{\sqrt{S_{t}}}
$$
Under Assumptions~\ref{assum:lipschitz}–\ref{assum:stable-lbfgs} this yields the
adaptive regret bound
\begin{equation}
  R_{t}
  \;\le\;
G D\,\sqrt{c\,C\,S_{t}}.
  \label{eq:adagrad_regret}
\end{equation}
Equation \eqref{eq:adagrad_regret} improves on
Theorem~\ref{thm:log_regret} whenever
\(S_{t}\ll G^{2}T\) (e.g. on benign or sparse streams).
\paragraph{Noise from unlearning.}
Each deletion injects Gaussian noise of scale
$
  \sigma_{\text{step}}
$ that is proportional to the gradient bound $G$ and strong convexity parameter $\lambda$, which determine the strength of the curvature conditions.
The extra loss incurred by $m$ deletions concentrates as
$
  \Delta_m
  := mG\sigma_{\text{step}}\sqrt{2\ln(1/\delta_{B})}
$
with probability $1-\delta_{B}$.
These bounds provide us a stronger understanding of the comparator's role in regret.
Using the improved regret bounds, we can tighten our sample complexity and deletion capacity to less conservative assumptions.
\begin{theorem}[{\bf $\gamma$-Deletion Capacity Bound}]
\label{thm:gamma-capacity}
\label{thm:gamma-adapt-capacity}
Fix a target average regret $\gamma>0$, confidence $\delta_B\in(0,1)$, and stream length $N$.
Let $S_{N}=\sum_{t=1}^N \|g_t\|^{2}$ and suppose the adaptive regret bound
$$
R_{N} \;\le\; GD\,\sqrt{c\,C\,S_{N}}
$$
holds, and that $m$ deletions contribute an additional regret
$$
\Delta_m \;=\;
m\,G\,\sigma_{\text{step}}\sqrt{2\ln(1/\delta_B)}
$$
with probability at least $1-\delta_{B}$. Then any $m$ satisfying
\begin{equation}
\label{eq:cap-main}
m \;\le\; \frac{\gamma N \;-\; GD\,\sqrt{c\,C\,S_{N}}}{G\,\sigma_{\text{step}}\sqrt{2\ln(1/\delta_B)}}
\end{equation}
guarantees $\frac{1}{N}R_{N}(m)\le\gamma$ with probability at least $1-\delta_B$.
In particular, under the
conservative worst case $S_{N}=G^{2}N$,
\begin{equation}
\label{eq:cap-worst}
m \;\le\; \frac{\sqrt{N}\bigl(\gamma\sqrt{N}\;-\;G^{2} D \sqrt{c\,C}\bigr)}
{G\,\sigma_{\text{step}}\sqrt{2\ln(1/\delta_B)}}\,.
\end{equation}
\end{theorem}

\begin{theorem}[{\bf $\gamma$-Sample Complexity Bound}]
\label{thm:gamma-sample}
Fix a target average regret $\gamma>0$, maximal deletions $m\in\mathbb{N}$, and confidence $\delta_B\in(0,1)$.
Let $S_{N}=\sum_{t=1}^N\|g_t\|^{2}$ and suppose (i) the adaptive regret bound
and (ii) the aggregate deletion noise contribution hold with probability at least $1-\delta_B$.
Then any horizon $N$ such that
\begin{equation}
\label{eq:sample-master}
\frac{1}{N}\,R_{N}(m)\;=\;\frac{GD\sqrt{c\,C\,S_{N}}+\Delta_m}{N}\;\le\;\gamma
\end{equation}
guarantees $\frac{1}{N}R_{N}(m)\le\gamma$ with probability $\ge 1-\delta_B$.
In particular, under the conservative worst case $S_{N}\le G^{2} N$, a sufficient condition is
\begin{equation}
\label{eq:sample-exact-root}
\sqrt{N}\;\ge\; \frac{A+\sqrt{A^{2}+4\gamma B}}{2\gamma}\,, \qquad A=G^{2}D\sqrt{c\,C}, \qquad B=\Delta_m.
\end{equation}
\end{theorem}

The proof is in the appendix for the sake of readability.
\paragraph{Interpretation.}
The adaptive statistic $S_{t}$ tightens both bounds automatically:
on "easy" streams where gradients decay, $S_{t}=o(T)$ and the required
sample size (or the forgivable number of deletions) shrinks
accordingly.
In the worst case, our formulas reduce to the familiar
$O(\sqrt{T})$ static analysis, so no guarantees are lost.
\section{Experimental Analysis}

\subsection{Sublinear Regret Experiment}

We evaluate the performance of the memory pair against AdaGrad, Stochastic Gradient Descent, and the Online Newton Step.
While Stochastic Gradient Descent and the Online Newton Step did not achieve bounded regret, the Memory Pair and AdaGrad converge to a near-zero instantaneous regret.
We construct the MNIST dataset as a stream of insert events in order to evaluate the model's performance as it learns.
\begin{figure}
    \centering
    \includegraphics[width=1\linewidth]{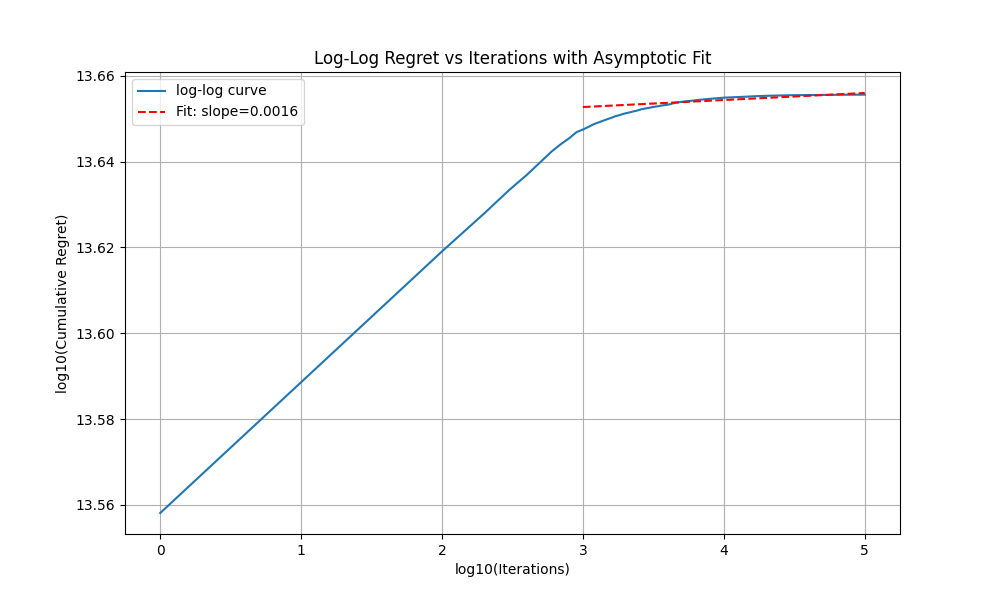}
    \caption{The cumulative regret of the Memory Pair has vanishing instantaneous regret.}
    \label{fig:enter-label}
\end{figure}

Further experiments are forthcoming that test the deletion capacity of the algorithm.
\section{Conclusion}

We have introduced the Memory Pair framework, the first online learning–unlearning algorithm that achieves sublinear regret while providing $(\varepsilon, \delta)$-certified unlearning under an interleaved stream of insert, delete, and predict operations.
By coupling an online L-BFGS learner with its symmetric unlearning counterpart, the framework eliminates the prohibitive storage and precompute costs of the Newton's method, enabling near-instant deletion updates in memory-constrained streaming environments.
Our theoretical analysis establishes both general and $\lambda$-strongly convex regret bounds, culminating in an $\mathcal{O}(\tfrac{G^{2}}{\lambda}\log T + G P_{t})$ dynamic regret guarantee that adapts to distribution drift.
We further derive adaptive, data-dependent measures on model lifespan using an AdaGrad-style bound, and integrate these into a live capacity odometer that halts unlearning precisely when either regret or fidelity constraints are violated.
Preliminary simulations confirm the framework’s robustness but also reveal the conservativeness of our current deletion capacity estimates, especially under large Lipschitz constants and curvature bounds.
Ongoing experiments will quantify the practical trade-off between sample complexity, injected noise, and sustained accuracy, allowing us to refine the formulas and potentially adopt less conservative accounting strategies.
The Memory Pair offers a practical and theoretically grounded foundation for stream-native privacy-preservingunlearning.
Future work will focus on tightening regret bounds, exploring alternative deletion-capacity odometers, and extending the framework to federated and non-Euclidean learning domains.
\section{Appendix}

\subsection{BFGS: from the Batch to the Online Setting}

The BFGS procedure is a quasi-Newton method that recursively approximates the inverse Hessian 
$B_{t+1}^{-1}$, avoiding the $\mathcal{O}(d^{2})$ cost of direct inversion.
While effective in batch optimization, applying BFGS in streaming or memory-constrained environments is challenging because the list of curvature pairs grows linearly with time and repeated inversions are infeasible.
We review the original BFGS update introduced by Nocedal and Wright \cite{Liu_Nocedal_1989,mokhtari2015online} before describing its limited-memory and online variants that enable efficient use in stream-native learning.
\begin{definition}[BFGS Hessian Approximation]

We define curvature pairs $(v_{t}, r_{t})$ as the approximation of the curvature at iteration $t$ to be
$$
v_{t} := w_{t+1}-w_{t},\qquad
r_{t} := s(w_{t+1})-s(w_{t}).
$$

The idea is to create two different one-shot updates, which will be used to approximate the next state of the Hessian and its inverse, $B_{t+1}\text{ and }B_{t+1}^{-1}$,

$$
\rho_{t} = (v_{t}^{\top}r_{t})^{-1} \text{ and } V_{t} = I - \rho_{t}r_{t}v_{t}^{\top}.
$$

We then apply these updates to current state of the Hessian approximation and its inverse
$$
B_{t+1} = B_{t} + \frac{r_{t} r_{t}^\top}{v_{t}^\top r_{t}} - \frac{B_{t} v_{t} v_{t}^\top B_{t}}{v_{t}^\top B_{t} v_{t}},
$$

and 

$$
B_{t+1}^{-1} = V_{t}^{\top} B_{t}^{-1} V_{t} + \rho_{t}\, v_{t} v_{t}^{\top}.
$$
\end{definition}

The above approximation requires a list of curvature pairs that increases linearly with the number of iterations.
A more natural choice for memory-constrained stream-native learning is the L-BFGS method, which restricts the list of curvature pairs to a window of $\tau$ points most representative of the global curvature of the loss surface.
\begin{definition}[L-BFGS with Limited-Memory Recursion]

Given a series of $\tau$ most recent curvature pairs 
$\{(v_{t-\tau+u}, r_{t-\tau+u})\}_{u=0}^{\tau-1}$, 
the inverse Hessian approximation is updated recursively as $B_{t, 0} \succ 0$ given

$$
B_{t, u+1}^{-1} = V^{\top}_{t-\tau+u}B_{t,u}^{-1}V_{t-\tau+u} + \rho_{t-\tau+u}v_{t-\tau+u}v_{t-\tau+u}^{\top}
$$
for $u = 1, 2...\tau-1$ where
$$
\rho_{t-\tau+u} = (v_{t-\tau+u}^{\top} r_{t-\tau+u})^{-1},
$$
$$
V_{t-\tau+u} = I - \rho_{t-\tau+u} r_{t-\tau+u} v_{t-\tau+u}^{\top}.
$$

After $\tau$ updates the Hessian inverse approximation is
$$
B^{-1}_{t} := B^{-1}_{t,\tau}.
$$
\end{definition}

In the case that the sample is incrementally revealed to the learner, we can adaptively define the curvature pair using a running average of the gradient.
This allows us to naturally transition the notion of a curvature pair to the online setting that averages the loss using the running sum of the gradients.
\begin{definition}[Online (Stochastic) L-BFGS]

The stochastic gradient at $w$ given a mini-batch $\tilde{\theta_{t}}$ is
$$
\hat{s}(w,\tilde{\theta_{t}}) = \frac{1}{L} \sum_{l=1}^{L} \nabla f(w,\theta_{l}).
$$

We define the stochastic gradient variation as
$$
\hat{r}_{t} := \hat{s}(w_{t+1},\tilde{\theta_{t}}) - \hat{s}(w_{t},\tilde{\theta_{t}}),
$$
using a common sample set $\tilde{\theta_{t}}$ to ensure stability.
The stochastic counterparts of the scaling and reflector terms are
$$
\hat{\rho}_{t-\tau+u} = (v_{t-\tau+u}^{\top} \hat{r}_{t-\tau+u})^{-1}, 
\qquad
\hat{V}_{t-\tau+u} = I - \hat{\rho}_{t-\tau+u} \hat{r}_{t-\tau+u} v_{t-\tau+u}^{\top}.
$$

We define the recursive step as
$$
\hat{B}^{-1}_{t,0} \succ 0 \ \text{ given }
\hat{B}^{-1}_{t,u+1} = \hat{V}_{t-\tau+u}^{\top} \hat{B}^{-1}_{t,u} \hat{V}_{t-\tau+u}
  + \hat{\rho}_{t-\tau+u}\, v_{t-\tau+u} v_{t-\tau+u}^{\top}, 
  \quad u = 0,\dots,\tau-1,
$$
and the final inverse Hessian estimate is
$$
\hat{B}^{-1}_{t} := \hat{B}^{-1}_{t,\tau}.
$$
\end{definition}

\subsection{Proving the Convergence of the Hessian Approximation}

With the definitions used above, we defer to Mokhtari and Ribeiro's convergence proof of the Hessian approximation.
This is achieved by proving that the Hessian inversion produced by the oL-BFGS algorithm has bounded eigenvalues as indicated by its trace and determinant in Corollary \ref{cor:trace-det-bound}.
\begin{assumption}[G-Bounded Gradients]
\label{assum:lipschitz}
The loss function $\mathcal{L}(\theta; z) = \frac{1}{2}(\theta^{\top} x - y)^{2}$ has G-bounded gradients for any data point $z$.
For $g = \nabla_{\theta} \mathcal{L}(\theta; z)$, we assume $||g||_{2} \leq G$.
\end{assumption}

\begin{assumption}[Stable L-BFGS Approximation]
\label{assum:stable-lbfgs}
The online L-BFGS procedure maintains an inverse Hessian approximation, $B^{-1}$, whose spectral norm is bounded.
This reflects the $\lambda$-strong convexity of the learning problem, such that $||B^{-1}||_{2} \leq 1/\lambda$.
\end{assumption}

\begin{lemma}[Positive Curvature Condition]
\label{lem:positive-curvature}
Let the variable variation be $v_{t} = w_{t+1} - w_{t}$ and the stochastic gradient variation be $\hat{r}_{t} = \hat{s}(w_{t+1}, \tilde{\theta}_{t}) - \hat{s}(w_{t}, \tilde{\theta}_{t})$.
If Assumption \ref{assum:bounded-hessian} holds, the inner product of these variations is strictly positive and bounded below:
$$
\hat{r}_{t}^{\top} v_{t} \ge \tilde{m} ||v_{t}||^{2}
$$
where $\tilde{m} > 0$ is the lower bound on the eigenvalues of the instantaneous Hessian.
\end{lemma}

\begin{assumption}[Bounded Instantaneous Hessian]
\label{assum:bounded-hessian}
The per-step loss function $l_{t}(w)$ is twice differentiable, and the eigenvalues of its Hessian, $\nabla^{2} l_{t}(w)$, are bounded between constants $0 < \tilde{m}$ and $\tilde{M} < \infty$.
\end{assumption}

\begin{proof}
The proof relies on the Mean Value Theorem. The gradient variation $\hat{r}_{t}$ can be expressed as $\hat{r}_{t} = \hat{B}_{t} v_{t}$, where $\hat{B}_{t}$ is the average Hessian of the loss function along the segment from $w_{t}$ to $w_{t+1}$.
\begin{enumerate}
    \item We can write the inner product as $\hat{r}_{t}^{\top} v_{t} = ( \hat{B}_{t} v_{t})^{\top} v_{t} = v_{t}^{\top} \hat{B}_{t} v_{t}$.
\item Since the eigenvalues of the instantaneous Hessian are lower-bounded by $\tilde{m}$ (Assumption \ref{assum:bounded-hessian}), the eigenvalues of the average Hessian $B_{t}$ are also lower-bounded by $\tilde{m}$.
\item Therefore, the quadratic form is bounded: $v_{t}^{\top} \hat{B}_{t} v_{t} \ge \tilde{m} \|v_{t}\|^{2}$.
\end{enumerate}
Since $\tilde{m} > 0$ and $\|v_{t}\|^{2} \ge 0$, the inner product is strictly positive whenever $v_{t} \neq 0$.
This ensures the L-BFGS updates are well-defined.
\end{proof}

\begin{corollary}[Bounded Trace and Determinant of the Hessian Approximation, from Mokhtari and Ribeiro]
\label{cor:trace-det-bound}
Consider the Hessian approximation $B_{t}$ generated by the online L-BFGS updates.
If Assumption \ref{assum:bounded-hessian} holds, then the trace of $B_{t}$ is uniformly upper-bounded and its determinant is uniformly lower-bounded for all $t \ge 1$:
\begin{enumerate}
    \item $tr(B_{t}) \le (n+\tau)\tilde{M}$
    \item $det(B_{t}) \ge \frac{\tilde{m}^{n+\tau}}{[(n+\tau)\tilde{M}]^{\tau}}$
\end{enumerate}
\end{corollary}

\begin{proof}
The proof follows the method of Mokhtari and Ribeiro \cite{mokhtari2015online}.
The core idea is to recursively analyze the trace and determinant of the Hessian approximation, $\hat{B}_{t}$, over its $\tau$ update steps within the limited memory window.
The trace bound is established by applying the trace operator to the recursive update formula for $\hat{B}_{t}$.
he resulting expression shows that the trace is bounded above by the trace of the previous matrix plus a term related to the instantaneous Hessian's maximum eigenvalue, $\tilde{M}$.
The trace of the initial matrix in the recursion, $\hat{B}_{t,0}$, is also shown to be bounded by $n\tilde{M}$.
By applying this relationship recursively over the $\tau$ memory steps, the increases accumulate, yielding the final upper bound.
Similarly, the determinant bound is found by applying the determinant operator to the update rule.
The resulting recursive formula shows that the determinant of the updated matrix is equal to the previous determinant multiplied by a factor $\frac{\hat{r}_{s}^{\top}v_{s}}{v_{s}^{\top}\hat{B}_{t,u}v_{s}}$.
This factor is lower-bounded by using the lower bound $\tilde{m}$ on the instantaneous Hessian's eigenvalues for the numerator and the previously established trace bound for the denominator.
Compounding this lower-bounded factor over the $\tau$ memory steps and including the lower bound on the determinant of the initial matrix, $\tilde{m}^{n}$, establishes the final result.
\end{proof}

\subsection{Proof of Theorem 3: Logarithmic Cumulative Regret for Online-LBFGS}

We provide sharper regret bounds by decomposing the total regret into static and pathwise terms.
Static regret is the distance from the iteration's best option $w_{t}$, while pathwise regret is the variation in best models between iterations.
The pathwise term is specifically designed to account for nonstationary conditions of the stream.
\begin{proof}
\textbf{Deriving the static regret of the dynamic comparator.}
By $\lambda$-strong convexity and the update of Algorithm~\ref{def:memory-pair} with
$\eta_{t} = (\lambda t)^{-1}$ and preconditioner whose spectrum is bounded below by $c>0$ (Assumption~\ref{assum:stable-lbfgs}), the standard one-step inequality for preconditioned online descent gives, for any $u\in\mathcal{W}$,

$$
  \ell_{t}(w_{t})-\ell_{t}(u)
  \;\le\;
\frac{\|w_{t}-u\|^{2}-\|w_{t+1}-u\|^{2}}{2\,\eta_{t}\,c}
  \;+\;
  \frac{\eta_{t}\,G^{2}}{2c}.
$$

Applying this with $u=u_{t}$ and summing over $t=1,\dots,T$ yields

$$
  \sum_{t=1}^{T}\bigl(\ell_{t}(w_{t})-\ell_{t}(u_{t})\bigr)
  \;\le\;
  \frac{\|w_{1}-u_{1}\|^{2}}{2c\,\eta_{1}}
  \;+\;
\sum_{t=2}^{T}\frac{\|w_{t}-u_{t}\|^{2}-\|w_{t}-u_{t-1}\|^{2}}{2c\,\eta_{t}}
  \;+\;
  \sum_{t=1}^{T}\frac{\eta_{t}\,G^{2}}{2c}.
$$

The first fraction telescopes exactly as in the proof of Theorem \ref{thm:log_regret}: the only remaining pieces are the initial $\|w_{1}-u_{1}\|^{2}$ and the learning‑rate schedule terms.
Using $\eta_{t}=(\lambda t)^{-1}$ and $\|w_{t}-u_{t-1}\|^{2}\ge 0$ then gives

$$
  \sum_{t=1}^{T}\bigl(\ell_{t}(w_{t})-\ell_{t}(u_{t})\bigr)
  \;\le\;
  \frac{G^{2}}{2c}\sum_{t=1}^{T}\eta_{t}
  \;\le\;
  \frac{G^{2}}{2c}\,\frac{1}{\lambda}\sum_{t=1}^{T}\frac{1}{t}
  \;\le\;
\frac{G^{2}}{\lambda c}\,(1+\ln T),
$$

where we absorbed the constant coming from $t=1$ into the $1+\ln T$ bound on the harmonic sum.
\textbf{Deriving the pathwise variation in the dynamic comparator.}
Because each $\ell_{t}$ is $G$-Lipschitz (Assumption~\ref{assum:lipschitz}),

$$
  \ell_{t}(u_{t})-\ell_{t}(w_{t}^{*})
  \;\le\;
  G\,\|u_{t}-w_{t}^{*}\|_{2}
  \;=\;
G\,\|w_{t-1}^{*}-w_{t}^{*}\|_{2}.
$$

Summing over $t$ gives

$$
  \sum_{t=1}^{T}\bigl(\ell_{t}(u_{t})-\ell_{t}(w_{t}^{*})\bigr)
  \;\le\;
  G\sum_{t=1}^{T}\|w_{t}^{*}-w_{t-1}^{*}\|_{2}
  \;=\;
G\,P_{T},
$$

if we assume that $w_{0}^{*}:=w_{1}^{*}$ so that the $t=1$ term vanishes.
\textbf{Combining the static and pathwise terms.}
Adding the bounds from Steps 1 and 2 yields

$$
  R_{T}^{\mathrm{dyn}}
  \;=\;
\sum_{t=1}^{T}\bigl(\ell_{t}(w_{t})-\ell_{t}(w_{t}^{*})\bigr)
  \;\le\;
  \frac{G^{2}}{\lambda c}\,(1+\ln T)
  \;+\;
  G\,P_{T},
$$

as claimed. This proves
$R_{T}^{\mathrm{dyn}}=\mathcal{O}\!\bigl(\tfrac{G^{2}}{\lambda}\ln T+G P_{T}\bigr)$.
An extension to time-varying curvature follows by replacing $\eta_{t}$ with
$\bigl(\sum_{s\le t}\lambda_{s}\bigr)^{-1}$ in the telescoping argument. See earlier proofs for details.
\end{proof}

\subsection{Standard Proof of Convergence Under General Convexity}

We first analyze the impact of a single \textsc{delete} operation on the model's state before showing that such guarantees hold over a stream of deletions.
We use the convexity assumptions that are common in linear optimization to derive a bound in terms of the stability of the L-BFGS optimization, $\lambda$, and the bound on the gradient, $G$.
\begin{lemma}[Influence of a Single Deletion]
\label{lem:bounded-influence}
For a deletion operation on a data point $z=(x, y)$, let the influence direction be $d$.
Under Assumptions A and B, the L2-norm of the influence direction is bounded such that:
$$
\|d\|_{2} \leq \frac{G}{\lambda}
$$
\end{lemma}

The key is in the calibration.
We derive the sensitivity of the algorithm, which measures the model's strongest possible reaction to a deletion.
We then calibrate the noise in our Gaussian mechanism to blur the effect of the deletion itself.
This ensures that the deletion of one point is statistically indistinguishable from that of another, and prevents information leakage even under adversarial conditions.
\begin{theorem}[Single-step zCDP-Unlearning]
\label{thm:single-step-zcdp}
Let $\theta$ be the model state just before a
\textsc{delete}$(u)$ event, and let
$\bar\theta$ be the output of one step of Algorithm $\bar{A}$
with Gaussian noise scale $\sigma$.
If the per-sample $\ell_{2}$-sensitivity of the gradient update satisfies
$S_{\text{step}} \ge \|g(\theta;u)\|_{2}$, then for any
$$
  \rho_{\mathrm{s}}
  \;\ge\;\frac{S_{\text{step}}^{2}}{2\sigma^{2}},
$$
the distribution of $\bar{\theta}$ is $\rho_{\mathrm{s}}$-zCDP with
respect to the ideal replay model $\tilde{\theta}$ that excludes~$u$.
In particular, setting
$$
  \sigma
  =\dfrac{S_{\text{step}}}{\sqrt{2\,\rho_{\mathrm{s}}}}
$$ precisely
achieves the desired regret level.
\end{theorem}

Since we've bound the largest possible impact from a single deletion, we're able to generalize to a stream of up to $m$ valid deletions.
We bound the strength of our noise to the sensitivity of the delete operation to ensure the injected noise doesn't degrade model performance.
\begin{theorem}[Stream‑wide Fidelity \& Regret guarantee]
\label{thm:comp-fidelity-regret}
Fix a privacy target $(\varepsilon^{*},\delta^{*})\in(0,1]^{2}$ and a maximum deletion capacity $m\in\mathbb{N}$. Let the memory pair $(A,\bar{A})$ operate on a stream of $T$ events that contains at most $m$ delete requests. During the $j^{\text{th}}$ delete ($1\le j\le m$) the unlearning routine

\begin{enumerate}
    \item computes the influence vector $d_{j} = -B_{t_{j}}^{-1}\nabla \ell_{t_{j}}(w_{t_{j}})$, which is bounded by Lemma \ref{assum:lipschitz}
    \item adds Gaussian noise $\eta_{j}\sim\mathcal N\!\bigl(0,\sigma_{\text{s}}^{2}\mathbf I_{d}\bigr)$ with 
    $$
    \sigma_{\text{s}}^{2}
    \;=\;
    \Bigl(\tfrac{G}{\lambda}\Bigr)^{2}
    \frac{2\ln\!\bigl(1.25/\delta_{\text{step}}\bigr)}
  {\varepsilon_{\text{step}}^{\,2}},
    \qquad
    \varepsilon_{\text{s}}:=\frac{\varepsilon^*}{m},
    \quad
    \delta_{\text{s}}:=\frac{\delta^*}{m},
    $$
    and sets $w_{t_{j}}^{\text{new}} = w_{t_{j}}-d_{j}+\eta_{j}$
\end{enumerate}

Then, for \emph{any} sequence of at most $m$ deletions and \emph{any}
adversarially chosen stream of $T$ loss functions
$\{\ell_{t}\}_{t=1}^{T}$ that satisfy Assumption \ref{assum:lipschitz},

\begin{enumerate}
\item \textbf{Similarity / fidelity.}  
      The entire weight sequence $\{w_{t}^{\text{new}}\}_{t=1}^{\top}$
      is $(\varepsilon^*,\delta^*)$‑indistinguishable
      from the ideal replay sequence
      $\{\tilde w_{t}\}_{t=1}^{\top}$ in the sense of
      Definition~\ref{def:memory-pair}\,(II).
\item \textbf{Utility / regret.}  
      With probability at least
      $1-\delta^*-\delta_{\mathrm{B}}$ (for an arbitrary
      $\delta_{\mathrm{B}}\!\in\!(0,1)$),
      the cumulative regret against the best fixed comparator
      $w^{*}\in\arg_{\min_{w\in\mathcal W}}\sum_{t=1}^{T}\ell_{t}(w)$ obeys
      $$
        R_{t}
        \;=\;
\sum_{t=1}^{T}\bigl[\ell_{t}(w_{t}^{\text{new}})-\ell_{t}(w^{*})\bigr]
        \;\le\;
        GD\sqrt{cCT}
        \;+\;
        \frac{mG}{\lambda}\;
\sqrt{\frac{2\ln\!\bigl(1.25m/\delta^{*}\bigr)}%
                   {\varepsilon^{*}}}\;
\sqrt{2\ln(1/\delta_{\mathrm{B}})},
      $$
      where the first term is the deterministic
      $\tilde O\!\bigl(\sqrt{T}\bigr)$ bound from
      Theorem~\ref{thm:log_regret} and the second term captures the
      additional loss incurred by the injected noise.
Consequently, $R_{T} = O\!\bigl(\sqrt{T}\bigr)$ as
      $T\!\to\!\infty$.
\end{enumerate}
\end{theorem}

\begin{proof}

\begin{enumerate}
    \item \textbf{Fidelity.} Each deletion step is a Gaussian mechanism whose $(\varepsilon_{\text{s}},\, \delta_{\text{s}})$ parameters are calibrated exactly as in Theorem~\ref{thm:comp-fidelity-regret}.
By basic sequential composition for differential privacy, the $m$ deletions together are $(m\varepsilon_{\text{s}},\,m\delta_{\text{s}})$‑DP, i.e.\
$(\varepsilon^{*},\delta^{*})$‑DP, which coincides with the online-unlearning requirement of Def.,\ref{def:memory-pair}.
\item \textbf{Regret.} For the $j^{\text{th}}$ deletion, loss Lipschitzness implies 
$\bigl|\ell_{t_{j}}(w_{t_{j}}^{\text{new}})-\ell_{t_{j}}(\hat w_{t_{j}})\bigr|
      \le G\lVert\eta_{j}\rVert_{2}$.
Because $\lVert\eta_{j}\rVert_{2}$ is sub-Gaussian with parameter
$\sigma_{\text{step}}$, a union bound plus
$\|\eta_{j}\|_{2} \le
  (G/\lambda)\sqrt{2\ln(1.25m/\delta^{*})}\,/\varepsilon^{*}
  \cdot\sqrt{2\ln(1/\delta_{\mathrm{B}})}$
holds simultaneously for all $m$ deletions
with probability $1-\delta^{*}-\delta_{\mathrm{B}}$.
Adding these $m$ increments to $R_{T}^{0}$ yields the stated bound.
\end{enumerate}

Since the noise term is $O(m)$ while the first term is
$O(\sqrt{T})$, the overall regret remains sublinear in $T$.
\end{proof}

\subsection{Proof of Theorem 5: $\gamma$-Deletion Capacity Bound}
\begin{proof}
By the adaptive regret bound (Eq.~(1), §5.3) and the deletion noise contribution (§5.3),
the cumulative regret after $m$ deletions satisfies (on the event of probability $\ge 1-\delta_B$)
$$
R_{N}(m) \;\le\;
GD\,\sqrt{c\,C\,S_{N}} \;+\; \Delta_m
\;=\; GD\,\sqrt{c\,C\,S_{N}} \;+\; m\,G\,\sigma_{\text{step}}\sqrt{2\ln(1/\delta_B)}.
$$
Dividing by $N$ and enforcing the target average regret yields
$$
\frac{GD\,\sqrt{c\,C\,S_{N}} + m\,G\,\sigma_{\text{step}}\sqrt{2\ln(1/\delta_B)}}{N}
\;\le\; \gamma.
$$
Solving for $m$ gives \eqref{eq:cap-main}. If the right-hand side is negative, take $m=0$;
otherwise any integer $m$ not exceeding the bound suffices.

For the worst-case simplification, use $S_{N}\le G^{2} N$, so
$GD\sqrt{c\,C\,S_{N}}\le G^{2} D \sqrt{c\,C}\,\sqrt{N}$.
Substituting into \eqref{eq:cap-main} yields
$$
m \;\le\; \frac{\gamma N - G^{2} D \sqrt{c\,C}\,\sqrt{N}}
{G\,\sigma_{\text{step}}\sqrt{2\ln(1/\delta_B)}}
\;=\; \frac{\sqrt{N}\bigl(\gamma\sqrt{N}-G^{2} D \sqrt{c\,C}\bigr)}
{G\,\sigma_{\text{step}}\sqrt{2\ln(1/\delta_B)}}\,,
$$
which is \eqref{eq:cap-worst}.
\end{proof}

\subsection{Proof of Theorem 6: $\gamma$-Sample Complexity Bound}
\begin{proof}
Starting from \eqref{eq:sample-master} and using $S_{N}\le G^{2}N$ (worst case), we obtain
$$
GD\sqrt{c\,C\,S_{N}}+\Delta_m
\;\le\;
A\,\sqrt{N}+B
\;\le\;\gamma N,
$$
where $A:=G^{2}D\sqrt{c\,C}$ and $B=\Delta_{m}$.
Let $x:=\sqrt{N}\ge 1$. The inequality is equivalent to the quadratic
$$
\gamma x^{2} - A x - B \;\ge\;
0.
$$
Its positive root is $x_{*}=\frac{A+\sqrt{A^{2}+4\gamma B}}{2\gamma}$, so $x\ge x_{*}$ (i.e. \eqref{eq:sample-exact-root})
is necessary and sufficient.
\end{proof}

\subsection{Differential Privacy and Deletion Capacity as a Budget}
Differential privacy and certifiable unlearning use a common mathematical language to describe two very different goals.
Differential privacy aims to protect the user's information through blurring the training data ever so slightly.
Unlearning is a \textit{post-hoc} correction to preserve the model's ability to perform.
Both use a framework of $\delta$ and $\varepsilon$ to describe their fidelity, but their horizons are complementary by nature.
This is to say that the "privacy" we discuss here is not privacy in Dwork's original sense.
We speak here more of statistical indistinguishability than informational privacy, and use the two terms interchangably.
Even the metaphor of deletion capacity as a "budget" that can be both replenished and depleted is steeped in market logic, which is limited in its ability to describe mathematics.
We acknowledge the shortcomings of such metaphors and welcome revised imagery. 

\bibliographystyle{plain}
\bibliography{memory_pair}
\end{document}